\def\year{2020}\relax
\documentclass[letterpaper]{article} 
\usepackage{aaai20}  
\usepackage{times}  
\usepackage{helvet} 
\usepackage{courier}  
\usepackage[hyphens]{url}  
\usepackage{graphicx} 
\usepackage{graphicx}  
\usepackage{subcaption}
\usepackage{complexity}
\usepackage{tikz}
\usetikzlibrary{shapes,arrows,positioning}
\usepackage{pgfplots}
\usepackage{stmaryrd}
\usepackage{amsmath}
\usepackage{mathtools}
\usepackage{bm}
\usepackage{amssymb}
\usepackage{amsthm}
\usepackage{microtype}
\newtheorem{dfn}{Definition}

\usepackage{algorithm}
\usepackage[noend]{algpseudocode}
\makeatletter
\def\BState{\State\hskip-\ALG@thistlm}
\makeatother
\newtheorem{thm}{Theorem}

\newtheorem{rmk}{Remark}
\newtheorem{exa}{Example}

\usepackage{xcolor}
\usepackage{soul}

\title{Approximate Weighted First-Order Model Counting: \\ Exploiting Fast Approximate Model Counters and Symmetry}
\author{Timothy van Bremen\\
KU Leuven, Belgium\\
timothy.vanbremen@cs.kuleuven.be
\And
Ond\v{r}ej Ku\v{z}elka\\
Czech Technical University in Prague, Czech Republic\\
ondrej.kuzelka@fel.cvut.cz
}

\begin{document}

\maketitle

\begin{abstract}
We study the symmetric weighted first-order model counting task and present \textsf{ApproxWFOMC}, a novel anytime method for efficiently bounding the weighted first-order model count in the presence of an unweighted first-order model counting oracle. The algorithm has applications to inference in a variety of first-order probabilistic representations, such as Markov logic networks and probabilistic logic programs. Crucially for many applications, we make no assumptions on the form of the input sentence. Instead, our algorithm makes use of the symmetry inherent in the problem by imposing cardinality constraints on the number of possible true groundings of a sentence's literals. Realising the first-order model counting oracle in practice using the approximate hashing-based model counter \textsf{ApproxMC3}, we show how our algorithm outperforms existing approximate and exact techniques for inference in first-order probabilistic models. We additionally provide PAC guarantees on the generated bounds.
\end{abstract}

\section{Introduction}

Given a propositional formula $\phi$, the \textit{model counting} problem asks for the number of models (satisfying assignments) of $\phi$. Model counting is the prototypical $\#\P$-complete problem. The \textit{weighted model counting} (WMC) problem generalizes this task by associating each assignment with a real-valued \textit{weight}, and asks for the weighted sum of the formula's models. In the past several years, the WMC task has attracted great interest as an ``assembly language'' for probabilistic inference, as inference in various formalisms such as graphical models \cite{DBLP:journals/ai/ChaviraD08} and probabilistic logic programming languages \cite{DBLP:journals/tplp/FierensBRSGTJR15} can be reduced to WMC. Many practical implementations of (weighted) model counters have also been introduced, such as \textsc{Dsharp}~\cite{DBLP:conf/ai/MuiseMBH12} and \textsc{minic2d}~\cite{DBLP:conf/ijcai/OztokD15}. In addition to exact weighted model counters, another line of research has unfolded among approximate model counters \cite{DBLP:conf/cp/ChakrabortyMV13,DBLP:conf/aaai/ChakrabortyFMSV14}, which are often capable of scaling to much larger problem sizes than exact methods. 

In practice, however, logical representations of real-world domains are often first-order, and thus are typically grounded into propositional logic before a weighted model counter can be used to infer probabilities. In general, a first-order probabilistic inference task can be reduced to an instance of the \textit{weighted first-order model counting} (WFOMC) problem, in which weights are assigned to interpretations of a first-order formula. In this paper, we consider the \textit{symmetric} WFOMC problem, where weights are associated with each predicate, as opposed to the asymmetric case where each possible grounding of a predicate may have a distinct weight.

WFOMC remains a difficult task. From a complexity point of view, \citeauthor{DBLP:conf/pods/BeameBGS15}~\shortcite{DBLP:conf/pods/BeameBGS15} showed that the data complexity of symmetric WFOMC for $\text{FO}^{k}$ ($k \geq 3$) is $\#\P_1$-hard, suggesting that in general sentences with at least three distinct logical variables are not domain-liftable.\footnote{In the artificial intelligence literature, a problem is said to be \textit{domain-liftable} if inference can be performed in polynomial time in the size of the domain \cite{DBLP:conf/nips/Broeck11}.} Nevertheless, the search for practical methods for performing WFOMC remains an active area of research. These methods can be divided into \textit{grounding} and \textit{non-grounding} approaches, depending on whether they rely on grounding out the input first-order formula to its propositional counterpart. Non-grounding algorithms operate directly on the first-order representation in order to circumvent the grounding step. Such implementations include \textsc{Forclift}~\cite{DBLP:conf/ijcai/BroeckTMDR11} and \textsc{Alchemy2}~\cite{DBLP:conf/uai/GogateD11a}. Other approaches require first grounding out the problem and then passing it to a propositional weighted model counter.

One question that has received relatively little attention is how one can efficiently exploit propositional model counters in practice for first-order problems: in other words, can we leverage off-the-shelf propositional model counters for WFOMC, in an efficient manner? In this paper, we answer this question in the affirmative, and show how such a strategy can be efficiently implemented using hashing-based approximate model counters. As we shall show later, existing hashing-based approximate model counting algorithms capable of dealing with weighted instances need an exponential number of \textsf{SAT} queries in the size of the domain when dealing with grounded first-order formulas. In order to overcome this, we first propose a decomposition of the WFOMC into the weighted sum of a number of (unweighted) first-order model counts of the input formula conjoined with cardinality constraints. These cardinality constraints serve to limit the number of true instances of the formula's atoms.
We then extend our approach to an anytime iterative algorithm that uses an intuitive search procedure to find dense regions in the space of weighted models. We evaluate our approach on first-order representations of probabilistic logic programs and Markov logic networks.


\section{Background}
In this section, we begin by explaining the principles of hashing-based approximate model counting techniques. We then briefly review the syntax of first-order logic, and formally define the weighted first-order model counting problem.
\subsection{Hashing-based Approximate Model Counting}
One promising approach to the model counting problem involves exploiting universal hash functions to get approximate counts. \citeauthor{DBLP:conf/cp/ChakrabortyMV13}~\shortcite{DBLP:conf/cp/ChakrabortyMV13} proposed an algorithm, \textsf{ApproxMC}, which uses XOR-based hash functions in order to obtain an approximate model count with arbitrary tolerance and confidence guarantees. The basic working principle of this approach involves adding a XOR constraint on a random subset of the variables appearing in the formula, which cuts the number of models approximately in half. After repeating this procedure a sufficient number of times, we may compute exactly the number of models in the constrained formula, and repeat this procedure a number of times to get a good sample of the size of an average ``cell''. Multiplying the median cell size by the number of cells created from imposing the XOR constraints then gives us an approximation of the overall model count.

This work later led to the development of even more efficient model counters using the same underlying principle. In particular, \textsf{ApproxMC2} \cite{DBLP:conf/ijcai/ChakrabortyMV16} was developed which reduced the number of calls needed to a \textsf{SAT} oracle to logarithmic in the number of variables of the input. Finally, the latest revision, \textsf{ApproxMC3} \cite{DBLP:conf/aaai/SoosM19}, was developed, which processes the constructed CNF-XOR formulas in a more efficient manner.

Crucial to all of these tools is that they give PAC guarantees on the resulting model count. We follow the notation of the papers above and denote by $R_F$ the set of models of a propositional formula $F$, and by $R_{F \downarrow S}$ the projection of $R_F$ onto a subset $S$ of variables in the formula.
\begin{thm}[\citeauthor{DBLP:conf/ijcai/ChakrabortyMV16} \citeyear{DBLP:conf/ijcai/ChakrabortyMV16}]\label{thm:basic}
Given a formula $F$, sampling set $S \subseteq \textsf{Vars}(F)$, a tolerance $\varepsilon > 0$, and a confidence $1 - \delta \in (0, 1]$, \textsf{ApproxMC3} returns a count $c$ such that $P(|R_{F \downarrow S}|/(1+\varepsilon) \leq c \leq (1 + \varepsilon)|R_{F \downarrow S}|) \geq 1 - \delta$. Moreover, the number of \textsf{SAT} oracle calls required is $k \in \mathcal{O}\left(\frac{\log(|S|)\log(\frac{1}{\delta})}{\varepsilon ^ 2}\right)$.
\end{thm}

The hashing-based approach was extended to \textsf{WISH} \cite{DBLP:conf/icml/ErmonGSS13} and \textsf{WeightMC} \cite{DBLP:conf/aaai/ChakrabortyFMSV14}, which each leverage related techniques to allow for weighted model counting. In the latter paper, the authors identify a parameter, \textit{tilt}, which is the ratio of the maximum weight of all satisfying assignments to the minimum weight of all satisfying assignments, and show that their procedure runs in time polynomial in the tilt of the input formula when equipped with a \textsf{SAT} oracle. However, the tilt of the grounding of a first-order formula can grow exponentially in the size of the domain, as illustrated in the example below.

\begin{exa}\label{coin}
Let $\Phi = \forall x . Heads(x) \lor Tails(x) \land \forall x . \neg Heads(x) \lor \neg Tails(x)$, and let $w(Heads) = 0.5$, $\bar{w}(Heads) = 1$, and $w(Tails) = 0.1$, $\bar{w}(Tails) = 1$. Let $\mathbf{D} = \{coin_1, \dots, coin_n\}$. Then $tilt(\Phi) = (\frac{0.5}{0.1})^n = 5^n$.
\end{exa}

Thus, using \textsf{WeightMC} on a first-order model may require an exponential number of \textsf{SAT} queries in the size of the domain. Although \citeauthor{DBLP:conf/aaai/ChakrabortyFMSV14} \shortcite{DBLP:conf/aaai/ChakrabortyFMSV14} also describe (in Section 6 of their paper) a way to theoretically reduce the runtime by adding constraints that split the space of solutions into regions with small enough tilt, they mention that this approach would require a pseudo-boolean solver capable of efficiently handling XOR constraints, so there is no practical implementation of this theoretical extension of \textsf{WeightMC}. In this paper we show how the number of \textsf{SAT} queries can be reduced to a number polynomial in the domain size, while still obtaining a practical algorithm that enables us to scale to problem instances that are too large for exact approaches and for which there currently exist no other practical methods with PAC guarantees.

\subsection{First-order Logic}
We deal with the function-free, finite domain fragment of first-order logic. An \textit{atom} of arity $n$ takes the form $P(t_1, \dots, t_n)$, where $P/n$ comes from a vocabulary of \textit{predicates}, and each argument $t_i$ is either a constant from a finite domain $\textbf{D}$, or a logical variable from a vocabulary of variables. A \textit{literal} is an atom or its negation. A \textit{formula} is formed by connecting one or more literals together using conjunction or disjunction. A formula may optionally be surrounded by one or more quantifiers of the form $\exists x$ or $\forall x$, where $x$ is a logical variable. A logical variable in a formula is said to be \textit{free} if it does not appear in any quantifier. A formula with no free variables is called a \textit{sentence}. A \textit{clause} is a sentence consisting of a disjunction of literals. A formula is in \textit{conjunctive normal form} (CNF) if it is the conjunction of one or more clauses containing only universal quantification.\footnote{We will see later how existential quantification can be dealt with using a Skolemization procedure by \citeauthor{DBLP:conf/kr/BroeckMD14}~\shortcite{DBLP:conf/kr/BroeckMD14}.} We follow the usual semantics of first-order logic.


\subsection{Weighted First-order Model Counting}
We begin by reviewing the definition of the first-order model count of a formula. Throughout this section, we fix a sentence $\phi$ containing predicates $P_1/r_1, \dots, P_k/r_k$.
\begin{dfn}
The first-order model count of $\phi$ over a domain of size $d$ is defined as:
\[ \mathsf{FOMC}(\phi, d) = |\mathsf{models}_d(\phi)| \]
where $\mathsf{models}_d(\phi)$ denotes the set of all models of $\phi$ under the domain $\mathbf{D} = \{1, \ldots, d\}$.
\end{dfn}
In order to define the weighted first-order model count of the formula, we must first define the notion of a weighting.
\begin{dfn}
Denote the set of predicates appearing in $\phi$ by $P_\phi$. A weighting on $\phi$ is a pair of mappings $w: P_\phi \rightarrow \mathbb{R}$ and $\bar{w}: P_\phi \rightarrow \mathbb{R}$.
\end{dfn}
\begin{dfn}
Let $(w, \bar{w})$ be a weighting on $\phi$. The weighted first-order model count of $\phi$ over a domain of size $d$ under $(w, \bar{w})$ is:
\begin{align*}
\mathsf{WFOMC}(\phi, d, w, \bar{w}) = \sum_{\omega \in \mathsf{models}_d(\phi)} \prod_{l \in \omega_{T}} w(\mathsf{pred}(l))\\ \prod_{l \in \omega_{F}} \bar{w}(\mathsf{pred}(l))
\end{align*}
where $\omega_{T}$ denotes the set of true predicates in the model $\omega$, and $\omega_{F}$ the false predicates. The notation $\mathsf{pred}(l)$ maps an atom $l$ to its corresponding predicate name.
\end{dfn}

\subsection{Cardinality Constraints}
\textit{Cardinality constraints} express bounds on the number of true instances of members of a set of propositions. In this paper, we are interested in expressing this constraint on the number of true groundings of a first-order predicate.

Various encodings of such a constraint are possible. When implementing our WFOMC algorithm in practice, we will be using \textsf{ApproxMC3} to provide an \textsf{FOMC} oracle by grounding out first-order formulas, so we will need to express this constraint in propositional form. For efficiency, we employ an encoding by \citeauthor{DBLP:conf/cp/BailleuxB03}~\shortcite{DBLP:conf/cp/BailleuxB03}. This encoding splits a cardinality constraint into two parts: a \textit{totalizer}, which counts the number of true propositions as a unary-encoded number, and a \textit{comparator}, which constrains this number to lie within the bounds specified. It adds $\mathcal{O}(n \log n)$ auxiliary variables and $\mathcal{O}(n^2)$ additional clauses of length at most 3, where $n$ is the size of the constrained variable set. Helpfully for our application, the encoding adds no extraneous models: the truth values of the auxiliary variables are uniquely determined by the state of the other variables (in other words, the auxiliary variables form a \textit{dependent support}, as we define in more detail later). However, even if one were to use an alternative encoding that does add extra models the use of projection in the model counting procedure would avoid any impact on our algorithm. This is explained further later in the paper.

\section{Algorithm}
In this section, we first show how the WFOMC of a sentence can be decomposed into a series of terms by making use of cardinality constraints. Although the utility of this decomposition is limited in practice, it forms the basis for the next section, where we show how we can further take advantage of cardinality constraints to design an iterative algorithm, \textsf{ApproxWFOMC}, that computes bounds for the WFOMC. We will first assume the existence of an \textsf{FOMC} oracle, and then show how this can be implemented in practice using the hashing-based approximate model counter \textsf{ApproxMC3}.

\subsection{An Exact Decomposition of the WFOMC}

We begin by giving the decomposition of a WFOMC problem into a sum of terms.
\begin{thm}\label{decomp}
Consider a sentence $\phi$ with predicates $P_1, \dots, P_k$. Then the WFOMC of $\phi$ can be decomposed into a weighted sum of first-order model counts as:
\begin{multline*}
\mathsf{WFOMC}(\phi, d, w, \bar{w})
\\ = \sum_{(n_1,\dots,n_k) \in \mathcal{K}} \prod_{i=1}^k \left[ w(P_i)^{n_i} \bar{w}(P_i)^{r_i-n_i} \cdot \textsf{FOMC}(\phi \land \phi^{CARD}_{(n_1,\dots,n_k)}, d) \right]
\end{multline*}
where $r_i = \textsf{arity}(P_i)^d$, $\mathcal{K} = \{(n_1, \dots, n_k) \in \mathbb{N}^{k}\}$ such that $n_i \in \lbrace 0, \dots, r_i \rbrace$, and $\phi^{CARD}_{(n_1,\dots,n_k)}$ denotes the first-order cardinality constraint fixing every model of $\phi$ to have exactly $n_i$ true instances of $P_i$.
\end{thm}
The intuition behind Theorem \ref{decomp} can be reasoned as follows: consider the case of a formula $\phi$ with a single predicate $P$, and suppose we add a cardinality constraint to $\phi$ to fix $P$ to have precisely $n$ true groundings. Then assuming a domain of size $d$, we must have there are $\binom{d}{n}$ different groundings each with the same weight of $w(P)^n\bar{w}(P)^{d-n}$. The formula above generalises this to multiple predicates. In practice, however, such a decomposition is typically too large to compute exactly, even though the number of terms grows polynomially in the size of the domain.

\begin{rmk}
The number of terms (and thus, \textsf{FOMC} oracle calls) in Theorem \ref{decomp} for a sentence $\phi$ over a domain of size $d$ is: \[ M(\phi, d) = \prod_{i=1}^{k} (d^{\textsf{arity}(P_i)} + 1) \]
\end{rmk}


\subsection{Approximating the WFOMC Using an Exact \textsf{FOMC} Oracle}


\paragraph{Overview} Our approach, \textsf{ApproxWFOMC}, to bounding the value of $\textsf{WFOMC}(\phi, d, w, \bar{w})$ is described in Algorithm \ref{approxwfomc}. One begins by obtaining the coarsest bounds possible for the WFOMC. This is done by computing the unweighted FOMC and multiplying by the weight obtained when all groundings of each predicate are true, or the case when all are false depending on which is larger. It is not difficult to see that this indeed gives valid bounds on the true weighted first-order model count.
\begin{exa}\label{coinstep1}
Consider again the coin toss example from Example \ref{coin}, and fix a domain of size $d = 6$. We have $\textsf{FOMC}(\phi) = 2^6 = 64$. Moreover, we know that the positive weights for both predicates are lower than their respective negative weights. Thus, we may compute the lower bound:
\begin{multline*}
    LB = w(Heads)^d \cdot w(Tails)^d \cdot \textsf{FOMC}(\phi) \\ = 0.5^6 \cdot 0.1^6 \cdot 64 = 10^{-6}.
\end{multline*}
and upper bound:
\begin{multline*}
    UB = \bar{w}(Heads)^d \cdot \bar{w}(Tails)^d \cdot \textsf{FOMC}(\phi) \\ = 1^6 \cdot 1^6 \cdot 64 = 64.
\end{multline*}
We therefore get the global bounds $(LB, UB) = (10^{-6}, 64)$ for the coarsest constraints possible $\{Heads \rightarrow (0, 6), Tails \rightarrow (0, 6)\}$.
\end{exa}
We then split the space by considering two possible cases for each weighted predicate: one where at most half of all groundings of the predicate are true, and one where at least half are true. Given $p$ weighted predicates, this means we split the space into $2^p$ parts.\footnote{Note that the number of weighted predicates is typically small for most relational models.} We can compute the FOMC for each part using cardinality constraints, and bounds on the maximum and minimum possible weights for these regions can also be computed accordingly. Then, the upper and lower bound for the WFOMC for each part can be stored in a queue that is sorted according to some heuristic function on these bounds. Most importantly, the upper bounds and the lower bounds of the two parts can be used to improve the upper and lower bounds $\textit{UB}$ and $\textit{LB}$ that we have for $\textsf{WFOMC}(\phi, d, w, \bar{w})$. Specifically, denoting the old upper and lower bounds (before splitting) as $\textit{u}$ and $\textit{l}$ and the $2^p$ new pairs of upper and lower bounds by $(l_1,u_1)$, $(l_2,u_2)$, $\dots$, $(l_{2^p},u_{2^p})$, we can update the upper bounds as $UB := UB - u + (u_1 + \dots + u_{2^p})$ and $LB := LB - l + (l_1 + \dots + l_{2^p})$.
\begin{exa}
We now take the constraints from Example \ref{coinstep1} and split it into 4 possible subconstraints: $c_1 = \{Heads \rightarrow (0, 3), Tails \rightarrow (0, 3)\}$, $c_2 = \{Heads \rightarrow (0, 3), Tails \rightarrow (4, 6)\}$, $c_3 = \{Heads \rightarrow (4, 6), Tails \rightarrow (0, 3)\}$ and $c_4 = \{Heads \rightarrow (4, 6), Tails \rightarrow (4, 6)\}$. Then imposing each of these cardinality constraints in turn gives us $\textsf{FOMC}(\phi \land \phi^{CARD}_{c_1}) = 20$, $\textsf{FOMC}(\phi \land \phi^{CARD}_{c_2}) = 22$, $\textsf{FOMC}(\phi \land \phi^{CARD}_{c_3}) = 22$, and $\textsf{FOMC}(\phi \land \phi^{CARD}_{c_4}) = 0$. We may now follow a similar process as that in the last example and compute upper and lower bounds for each of these non-overlapping regions, and push these bounds along with their respective constraints onto a queue. We now also update our global bounds $(LB, UB)$ on the WFOMC: suppose we compute the bounds $(l_i, u_i)$ for each constraint $c_i$. Then we can tighten our bounds from $(10^{-6}, 64)$ to $(l_1 + l_2 + l_3 + l_4, u_1 + u_2 + u_3 + u_4) = (2.5 \times 10^{-3} + 3.4375 \times 10^{-4} + 2.75 \times 10^{-6} + 0, 20 + 1.375 + 2.2 \times 10^{-3} + 0) = (0.0028465, 21.3772)$.
\end{exa}
The first element is then popped from the queue, and the procedure repeats until the bounds are sufficiently tight.

\paragraph{Details} The pseudocode shown in Algorithm \ref{approxwfomc} uses several operations that have not been described yet. We provide the details here. The function $\textsf{WeightedPredicates}(\phi)$ returns the set of all non-neutral predicates in $\phi$ (i.e.\ all predicates having a positive or negative weight other than 1). The procedure \textsf{DictProduct} refers to a Cartesian product of dictionaries: given a dictionary $D$ whose values are lists, it returns the list of all dictionaries such that each key $k$ is a value in the list $D[k]$. For example, given the dictionary $d = \{\textsf{foo} \rightarrow [1, 2], \textsf{bar} \rightarrow [a, b]\}$, we have $\textsf{DictProduct}(d) = [\{\textsf{foo} \rightarrow 1, \textsf{bar} \rightarrow a\}, \{\textsf{foo} \rightarrow 1, \textsf{bar} \rightarrow b\}, \{\textsf{foo} \rightarrow 2, \textsf{bar} \rightarrow a\}, \{\textsf{foo} \rightarrow 2, \textsf{bar} \rightarrow b\}]$. The notation $\phi^{CARD}_{a}$ denotes the first-order formula imposing the cardinality constraints contained in the dictionary $a$. For example, if $a = \{ P_1 \rightarrow (0,2), P_2 \rightarrow (0,1) \}$, then $\phi^{CARD}_{a}$ would impose the constraint that predicates $P_1$ and $P_2$ have at most one and two true groundings respectively.

The priority queue is sorted in decreasing order according to a heuristic function on the elements: given, a tuple $(constraints, lb, ub)$, its heuristic is computed as $|ub - lb|$. It thus splits regions with the largest gap between upper and lower bounds first. Alternative heuristics are possible: for example, one could process regions with the greatest overall model count first, ignoring bounds on the weights. However, in practice we find our heuristic works well.

\begin{algorithm}[!h]
\caption{\textsf{ApproxWFOMC}}\label{approxwfomc}
\hspace*{\algorithmicindent} \textbf{Input} First-order CNF $\phi$, weights $(w, \bar{w})$, domain size $d$, tolerance $\tau$\\
\hspace*{\algorithmicindent} \textbf{Output} $(b_1, b_2)$ such that $b_1 \leq \mathsf{WFOMC}(\phi, d, w, \bar{w}) \leq b_2$ and $\frac{b_2}{b_1} < 1 + \tau$
\begin{algorithmic}[1]
\BState \emph{{\bf /* Initialization */}}:
\State $\textit{queue} \gets \text{new priority queue}$
\State $\textit{LB},\ \textit{UB} \gets \textsf{FOMC}(\phi, d)$
\For{$P$ in $\textsf{WeightedPredicates}(\phi)$}
    \State $\xi \gets d^{\textsf{arity}(P)}$
    \State $\textit{LB} \gets \textit{LB}\cdot \min(w(P)^\xi, \bar{w}(P)^\xi)$
    \State $\textit{UB} \gets \textit{UB}\cdot \max(w(P)^\xi, \bar{w}(P)^\xi)$
	\State $\textit{constraints}[P] \gets (0, \xi)$
\EndFor
\State $\text{Store } (\textit{constraints}, \textit{newLb}, \textit{newUb})$ in $\textit{queue}$
\BState \emph{{\bf /* Main loop */}}:
\While {$\frac{newUb}{newLb} \geq 1 + \tau$ and $\textit{queue}$ is non-empty}
	\State Pop $(\textit{constraints}, \textit{lb}, \textit{ub})$ from $\textit{queue}$
	\State{/* {\em Constructing refined constraints (splitting)} */}
	\If{$\textit{constraints}$ cannot be decomposed further}
		\State {\bf continue}
	\EndIf
	\State $\textit{newConstr} \gets \{\}$
	\For{$P$ in $\textsf{WeightedPredicates}(\phi)$}
		\State $(l, u) \gets constraints[P]$
		\If{$\l \neq u$}
			\State $newConstr[P] \gets \{(l, \lfloor \frac{l+u}{2} \rfloor), (\lfloor \frac{l+u}{2} \rfloor + 1, u)\}$
		\Else
			\State $newConstr[P] \gets \{constraints[P]\}$
		\EndIf					
	\EndFor
	
	\State{/* {\em Recomputing $\textit{LB}$ and $\textit{UB}$ using $\textit{newConstr}$} */}
	\State $\textit{LB} \gets \textit{LB} - lb$
	\State $\textit{UB} \gets \textit{UB} - ub$

	\For{$\textit{refinedConstr}$ in $\textsf{DictProduct}(\textit{newConstr})$}
		\State $\textit{tMin},\ \textit{tMax} \gets 1$
		\For{$P$ in $\textsf{WeightedPredicates}(\phi)$}
		    \State $\xi \gets d^{\textsf{arity}(P)}$
		    \State $(l, u) \gets refinedConstr[P]$    		
		    \State $\textit{tMin} \gets \textit{tMin}\cdot \min(w(P)^{l}\bar{w}(P)^{\xi-l}, w(P)^{u}\bar{w}(P)^{\xi-u})$
    		\State $\textit{tMax} \gets \textit{tMax}\cdot \max(w(P)^{l}\bar{w}(P)^{\xi-l}, w(P)^{u}\bar{w}(P)^{\xi-u})$
		\EndFor
		\State $\textit{mc} \gets \textsf{FOMC}(\phi \land \phi^{CARD}_{\textit{refinedConstr}}, d)$
		\State $\textit{LB} \gets \textit{LB} + \textit{tMin}\cdot mc$
		\State $\textit{UB} \gets \textit{UB} + \textit{tMax}\cdot mc$
		\State Push $(refinedConstr, \textit{tMin}\cdot mc, \textit{tMax}\cdot mc)$ to $\textit{queue}$
	\EndFor
\EndWhile
\Return $(\textit{LB}, \textit{UB})$
\end{algorithmic}
\end{algorithm}

\subsection{Approximating the WFOMC Using an Approximate \textsf{FOMC} Oracle}

In practice we may only have access to an approximate \textsf{FOMC} oracle rather than an exact one: for example, we may wish to ground the input sentence and use \textsf{ApproxMC3} to provide such an oracle. In this case, in order to provide $\varepsilon$-$\delta$ style guarantees in \textsf{ApproxWFOMC}, we need to set the correct parameters to \textsf{ApproxMC3}.

\begin{thm}
Given a sentence $\phi$, let $(LB, UB) = \textsf{ApproxWFOMC}(\phi, w, \bar{w}, d, \tau)$. Suppose each \textsf{FOMC} oracle call is made by grounding the problem and calling \textsf{ApproxMC3} with tolerance $\varepsilon$ and confidence $\delta_i$. Then we have:
\[ \textit{Pr} \left[ \frac{\textit{LB}}{ 1 + \varepsilon } \leq \textsf{WFOMC} \left( \phi, w, \bar{w}, d \right) \leq \textit{UB} \left( 1 + \varepsilon \right) \right] \geq 1 - \delta \]
where $\delta = \sum_i \delta_i$.
\end{thm}
\begin{proof}
Let $M$ denote the number of calls to $\textsf{ApproxMC3}$ made by $\textsf{ApproxWFOMC}$ and let $c_i$ denote the output of the $i$-th call to $\textsf{ApproxMC3}$. Observe that at any point in $\textsf{ApproxWFOMC}$'s run both $\textit{LB}$ and $\textit{UB}$ are weighted sums of the outputs of $\textsf{ApproxMC3}$: $\textit{UB} = \sum_{i=1}^M \gamma_i \cdot c_i$ and $\textit{LB} = \sum_{i=1}^M \gamma_i' \cdot c_i$ (the values of the coefficients $\gamma_i$ and $\gamma_i'$ are not important for the purposes of the proof). Next let $c_i^*$ denote the true model count corresponding to the approximate value $c_i$ returned by $\textsf{ApproxMC3}$ and let $\textit{UB}^* = \sum_{i=1}^M \gamma_i \cdot c_i^*$ and $\textit{LB}^* = \sum_{i=1}^M \gamma_i' \cdot c_i^*$ be the respective bounds returned by $\textsf{ApproxWFOMC}$. It follows from the guarantees on $\textsf{ApproxMC3}$ (Theorem \ref{thm:basic}) that the probability that $c_i \not\in [\frac{c_i^*}{1+\varepsilon}, (1+\varepsilon) \cdot c_i^*]$ is no greater than $\delta_i$. Then by the union bound, we have that the probability that at least one $c_i \not\in [\frac{c_i^*}{1+\varepsilon}, (1+\varepsilon) \cdot c_i^*]$ is at most $\delta = \sum_{i=1}^M \delta_i$. Hence, with probability at least $1-\delta$ it holds $\textit{LB} \leq (1+\varepsilon) \textit{LB}^*$ and $\frac{\textit{UB}^*}{1+\varepsilon} \leq \textit{UB}$ from which we then have $\frac{\textit{LB}}{(1+\varepsilon)} \leq \textit{LB}^*$ and $\textit{UB}^* \leq (1+\varepsilon) \textit{UB}$. Next it follows from a simple inspection of the pseudocode of $\textsf{ApproxWFOMC}$ that $\textit{LB}^* \leq \textsf{WFOMC} \left( \phi, w, \bar{w}, d \right) \leq \textit{UB}^*$. Using the probabilistic bounds just derived for $\textit{LB}$ and $\textit{UB}$, we obtain that $\frac{\textit{LB}}{ 1 + \varepsilon } \leq \textsf{WFOMC} \left( \phi, w, \bar{w}, d \right) \leq \textit{UB} \left( 1 + \varepsilon \right)$ with probability at least $1-\delta$.
\end{proof}

One remaining question is how to set the values of the $\delta_i$'s. One possibility is to set $\delta_i := \delta/{M}_{\textit{max}}$ where ${M}_{\textit{max}}$ is the theoretical maximum number of calls to \textsf{ApproxMC3}, and $\delta$ is a confidence parameter set by the user. Another possibility that may be preferable in practice is to set $\delta_i := \delta /(i \cdot \ln{(M_\textit{max}+1)})$. Since the sum of the first $M_\textit{max}$ elements of the harmonic series is at most $1/\ln(M_\textit{max}+1)$, we will always have $\sum_{i=1}^{M} \delta_i \leq \delta$ for $M \leq M_\textit{max}$. Here we are exploiting the fact that our algorithm will often use much fewer calls to \textsf{ApproxMC3} than $M_\textit{max}$. Let $M$ be the number of calls to \textsf{ApproxMC3} made by \textsf{ApproxWFOMC}. Since, as asserted by Theorem \ref{thm:basic}, the number of calls to a \textsf{SAT} oracle made by \textsf{ApproxMC3} is $\mathcal{O}\left(\frac{\log(|S|)\log(\frac{1}{\delta})}{\varepsilon ^ 2}\right)$, \textsf{ApproxWFOMC} will need $\mathcal{O}\left(M \cdot \frac{\log(|S|) \cdot (\log(\frac{1}{\delta}) + \log M + \log (\ln(M_\textit{max}+1)))}{\varepsilon ^ 2}\right)$ calls to a \textsf{SAT} oracle.

\section{Implementation and Experiments}
We have implemented our algorithm and tested it on encodings of Markov logic networks (MLNs) \cite{DBLP:journals/ml/RichardsonD06} and ProbLog programs \cite{DBLP:journals/tplp/FierensBRSGTJR15}. We implement the \textsf{FOMC} oracle by using \textsf{ApproxMC3} as described earlier. In this section, we briefly review each first-order model we perform our experiments on and show how each can be cast as a WFOMC task, in both cases following the encodings of \citeauthor{DBLP:conf/kr/BroeckMD14}~\shortcite{DBLP:conf/kr/BroeckMD14}. We follow with an analysis of our experimental results.

\begin{figure*}[!h]
\begin{subfigure}{.3\textwidth}
\centering
		\begin{tikzpicture}[scale=0.7]
			\begin{semilogyaxis}[
			    xlabel={Domain size},
			    ylabel={Runtime (s)},
			    xmin=2, xmax=6,
			    ymin=0, ymax=100000,
			    xtick={2,3,4,5,6},
			    ytick={0.01,0.1,1,10,100,1000,10000,100000},
			    legend pos=north west,
			    ymajorgrids=true,
			    grid style=dashed,
			]
			 
			\addplot[
			    color=blue,
			    mark=square,
			    ]
			    coordinates {
			    (2,0.016)(3,0.40)(4,53.10) 
			    };
			    \addlegendentry{SDD}
		   	\addplot[
			    color=red,
			    mark=triangle,
			    ]
			    coordinates {
			    (2,0.27*7)(3,3.14*8.47)(4,30.84*9.45)(5,251.07*11)(6,2809.45*12.4)
			    };
			    \addlegendentry{\textsf{ApproxWFOMC}}
			 
			\end{semilogyaxis}
		\end{tikzpicture}%
\caption{Runtime of various WFOMC methods for the \texttt{transitive-smokers-mln} problem for various domain sizes.}
\label{fig:domainscaletime-mln}
\end{subfigure}%
\hfill\begin{subfigure}{.3\textwidth}
\centering
		\begin{tikzpicture}[scale=0.7]
			\begin{axis}[
			    xlabel={Domain size},
			    ylabel={Number of \textsf{ApproxMC3} calls},
			    xmin=2, xmax=6,
			    ymin=0, ymax=250,
			    xtick={2,3,4,5,6},
			    ytick={50,100,150,200,250},
			    legend pos=north west,
			    ymajorgrids=true,
			    grid style=dashed,
			]
			 
			\addplot[
			    color=blue,
			    mark=square,
			    ]
			    coordinates {
			    (2,23)(3,43)(4,85)(5,141)(6,229)
			    };
			    \addlegendentry{\texttt{transitive-smokers-mln}}
			\addplot[
			    color=red,
			    mark=triangle,
			    ]
			    coordinates {
			    (2,11)(3,13)(4,21)(5,23)(6,33)
			    };
			    \addlegendentry{\texttt{conference-problog}}
			 
			\end{axis}
		\end{tikzpicture}%
\caption{Number of \textsf{ApproxMC3} calls made by \textsf{ApproxWFOMC} with $\tau = 0.2$.} 
\label{fig:domainscalefomc}
\end{subfigure}%
\hfill\begin{subfigure}{.3\textwidth}
\centering
		\begin{tikzpicture}[scale=0.7]
			\begin{semilogyaxis}[
			    xlabel={Domain size},
			    ylabel={Ratio},
			    xmin=2, xmax=6,
			    ymin=0.001, ymax=10,
			    xtick={2,3,4,5,6},
			    ytick={0.001,0.01,0.1,1,10},
			    legend pos=north west,
			    ymajorgrids=true,
			    grid style=dashed,
			]
			 
			\addplot[
			    color=blue,
			    mark=square,
			    ]
			    coordinates {
			    (2,0.17037037)(3,0.038392857)(4,0.015384615)(5,0.007173382)(6,0.004074516) 
			    };
			    \addlegendentry{\texttt{transitive-smokers-mln}}
			\addplot[
			    color=red,
			    mark=triangle,
			    ]
			    coordinates {
			    (2,1.222222222)(3,0.8125)(4,0.84)(5,0.638888889)(6,0.673469388) 
			    };
			    \addlegendentry{\texttt{conference-problog}}
			 
			\end{semilogyaxis}
		\end{tikzpicture}%
\caption{Ratio of \textsf{ApproxMC3} calls made by \textsf{ApproxWFOMC} to \textsf{FOMC} oracle calls in the decomposition in Theorem \ref{decomp}.}
\label{fig:ratio}
\end{subfigure}
\caption{}
\end{figure*}

\subsection{Encoding an MLN}
Recall that an MLN comprises a set of tuples $(w, \phi)$, where $w$ is a real-valued weight and and $\phi$ is a first-order formula. For example, consider the MLN below, ``\texttt{transitive-smokers-mln}'':
\begin{align*}
1.22 \quad &\texttt{stress}(X) \rightarrow \texttt{smokes}(X)\\
2.08 \quad &\texttt{friends}(X,Y) \land \texttt{smokes}(X) \rightarrow \texttt{smokes}(Y)\\
0.69 \quad &\texttt{friends}(X,Y) \land \texttt{friends}(X,Z) \rightarrow \texttt{friends}(X,Z)
\end{align*}

The first rule states that people who are stressed are likely to smoke. The second states that smokers tend to make friends with other smokers. The last rule states that the $\texttt{friends}$ relation is typically transitive.

\begin{dfn}
The WFOMC encoding $(\Delta, w, \bar{w})$ of an MLN is constructed as follows: for each tuple $(w_i, \phi_i(\bm{x}_i))$ in the MLN, where $\bm{x}_i$ denotes the free logical variables occurring in $\phi_i$, we introduce an auxiliary predicate $P_i/|\bm{x}_i|$. Then for each formula in the MLN, $\Delta$ is formed by conjoining the sentences $\forall \bm{x}_i P_i \leftrightarrow \phi_i(\bm{x}_i)$. The weighting is defined by setting $w(P_i) = e^{w_i}$, $\bar{w}(P_i) = 1$, and $w(Q) = \bar{w}(Q) = 1$ for all other predicates $Q$.
\end{dfn}

The encoding of the first rule of \texttt{transitive-smokers-mln} earlier is therefore:
\begin{align*}
&\forall X P_1(X) \leftrightarrow (\texttt{stress}(X) \rightarrow \texttt{smokes}(X))
\end{align*}
with $w(P_1) = e^{2.9}$, $\bar{w}(P_1) = 1$, and $w(\texttt{stress}) = w(\texttt{smokes}) = \bar{w}(\texttt{stress}) = \bar{w}(\texttt{smokes}) = 1$.


We can take advantage of some domain-specific knowledge of the MLN encoding in order to further optimize our algorithm when computing the partition function of an MLN. We first recall the definition of an (in)dependent support \cite{DBLP:journals/constraints/IvriiMMV16}.

\begin{dfn}
Let $F$ denote a propositional formula, and let $X$ denote the set of variables appearing in $F$. Then $I \subseteq X$ is said to be an independent support of $F$ if, for any two models $\sigma_1, \sigma_2 \in R_F$ that agree on $I$, we have $\sigma_1 = \sigma_2$. In other words, the truth values of $I$ uniquely determine the truth value of every variable in $X \setminus I$. The remaining variables $X \setminus I$ are called a dependent support.
\end{dfn}


\begin{rmk}\label{mis}
Let $(\Delta, w, \bar{w})$ denote the WFOMC encoding of an MLN. Then, after grounding $\Delta$ over some domain $\bm{D}$, the ground instances of all non-auxiliary predicates form an independent support for the grounding of $\Delta$.
\end{rmk}

Based on the observation in Lemma \ref{mis}, we may pass the ground instances of all non-auxiliary predicates as a \textit{sampling set} to \textsf{ApproxMC3} in every call in our algorithm, and perform projected model counting. This has the effect of shortening the XOR constraints which must be processed by \textsf{ApproxMC3}, and provides a significant speed-up to the model counting times.

Finally, observe that all non-auxiliary predicates have neutral weight, so constraints will only be imposed on the auxiliary predicates.

\subsection{Encoding a Probabilistic Logic Program}
A probabilistic logic program combines a classical logic program with uncertainty. In this paper, we use the ProbLog language, and refer the reader to \citeauthor{DBLP:journals/tplp/FierensBRSGTJR15}~\shortcite{DBLP:journals/tplp/FierensBRSGTJR15} for a complete overview of the syntax and semantics. A ProbLog program is a set of probabilistic facts of the form $p \texttt{::} a$, along with a classical logic program $\phi$ whose rule heads do not contain facts from $F$. For example, consider the program ``\texttt{conference-problog}'' below:
\begin{align*}
0.1\ &\texttt{::}\ \texttt{Attends}(X).\\
0.3\ &\texttt{::}\ \texttt{ToSeries}(X).\\
\texttt{Series}\ &\texttt{:-}\ \texttt{Attends}(X), \texttt{ToSeries}(X).
\end{align*}
%

\begin{dfn}
The WFOMC encoding $(\Delta, w, \bar{w})$ of a tight ProbLog program $\gamma$ is constructed by setting $\Delta$ as Clark's completion of the program, setting $w(\textsf{pred}(\alpha)) = P(\alpha), \bar{w}(\alpha) = 1 - P(\alpha)$ for each probabilistic fact $p \texttt{::} a$, and setting the weight of all other predicates to neutral.
\end{dfn}

Note that the transformation above comes with some caveats: in particular, the formula corresponding to our example program \texttt{conference-problog} contains existential quantification, but we assume that the CNF provided to \textsf{ApproxWFOMC} contains only universally quantified clauses. To overcome this, we employ the Skolemization procedure of \citeauthor{DBLP:conf/kr/BroeckMD14}~\shortcite{DBLP:conf/kr/BroeckMD14}. However, this creates negative weights in the CNF, causing problems with the computation of the lower and upper bounds in the algorithm. To avoid this issue, we may simply treat the Skolem predicate whose negative weight is $-1$ as a neutral predicate, thus not taking it into account in the computation of the upper and lower bounds. Moreover, we project away the auxiliary atoms resulting from the Skolem and Tseitin predicates, avoiding any extraneous models they create.

\subsection{Experimental Results}

We tested \textsf{ApproxWFOMC} on the WFOMC encodings of the \texttt{transitive-smokers-mln} program and \texttt{conference-problog}. We set out to answer the following questions:
\begin{enumerate}
    \item How does the performance of \textsf{ApproxWFOMC} on first-order probabilistic models compare to solving the same problem using exact knowledge compilation (SDDs)?
    \item How does the number of \textsf{FOMC} oracle calls needed by \textsf{ApproxWFOMC} scale with the domain size?
    \item How significant of an improvement does the search method proposed by \textsf{ApproxWFOMC} yield over the decomposition in Theorem \ref{decomp}, in terms of the number of \textsf{FOMC} oracle calls?
\end{enumerate}


\noindent We investigate each question individually.

\paragraph*{Q1} In Figure \ref{fig:domainscaletime-mln}, we show how the domain size affects the runtime of \textsf{ApproxWFOMC} and compare it to the SDD library \cite{DBLP:conf/aaai/ChoiD13}, called via the wrapper PySDD. Although SDDs outperform \textsf{ApproxWFOMC} with domain size 2, 3 and 4, \textsf{ApproxWFOMC} performs better with larger domains, with SDD compilation running out of memory already with a domain of size 5.

\paragraph*{Q2} In Figure \ref{fig:domainscalefomc}, we show how the number of \textsf{ApproxMC3} calls made by \textsf{ApproxWFOMC} increases as the domain size scales for a fixed tolerance value of 0.2. We observe that the number of calls grows quicker in the domain size for the \texttt{transitive-smokers-mln}, due to the larger number of predicates with higher arities appearing in this problem.

Notice that, despite a modest increase in the number of \textsf{ApproxMC3} calls when the domain size goes from 5 to 6, we see a significant increase in runtime for \textsf{ApproxWFOMC}. Thus, each approximate model counter call takes longer to return, especially as the problems become increasingly constrained with tighter cardinality bounds. This leads to the natural question of how cardinality constraints can be more efficiently incorporated into model counters, which we discuss further in the next section.

\paragraph*{Q3} In Figure \ref{fig:ratio}, we show the efficiency gain of \textsf{ApproxWFOMC} over using a simple decomposition of the form in Theorem \ref{decomp}, by quantifying the ratio between the number of \textsf{FOMC} oracle calls made by \textsf{ApproxWFOMC} to the number needed in the decomposition. We see that, in the case of more challenging MLN problem, the ``efficiency ratio'' improves significantly as domain size increases. The effect is less clear for the comparatively simpler encoding of \texttt{conferences-problog}.





\section{Further Work}
There are still many avenues for further research. One barrier in particular is the performance of \textsf{ApproxMC3} on highly constrained formulas, and in particular those we observe in our setting with highly restrictive cardinality constraints. This raises the question of whether it is possible to integrate support for these constraints into the model counter itself. The phase transition behaviour of 1-CARD-XOR formulas (the conjunction of a number of XOR clauses with a single cardinality constraint) was recently investigated by \citeauthor{DBLP:conf/ijcai/PoteJM19}~\shortcite{DBLP:conf/ijcai/PoteJM19}, paving a theoretical foundation for understanding the runtime of existing solvers. However, to the best of our knowledge, there currently exists no specialised solver for handling CNF formulas subject to both cardinality and XOR constraints. If this gap was filled and integrated into \textsf{ApproxMC3}, we could see significant gains in the performance of our approach.


\section{Conclusion}
We introduced \textsf{ApproxWFOMC}, an anytime approximate WFOMC algorithm with PAC guarantees, and showed how it can be applied to inference in MLNs and probabilistic logic programs. Initial results are promising, showing that it is able to scale to domain sizes that are too large for existing exact methods.

\vspace{0.2cm}
\noindent{\bf Acknowledgments.} TvB's work is supported by the Research Foundation -- Flanders (grant G095917N). OK's work has been supported by the OP VVV project {\it CZ.02.1.01/0.0/0.0/16\_019/0000765} ``Research Center for Informatics'' and a donation from X-Order Lab. Part of this work was done while OK was already supported by the Czech Science Foundation project ``Generative Relational Models'' (20-19104Y). Both authors thank Luc De Raedt for helpful feedback on this paper. 

\bibliography{starai-paper}
\bibliographystyle{aaai}
\end{document}